\documentclass[letterpaper,10pt, conference]{ieeeconf}

\usepackage{setspace}
\usepackage{url}
\usepackage[utf8]{inputenc}
\usepackage[T1]{fontenc}
\usepackage{times}
\usepackage{graphicx}
\usepackage{wrapfig}
\usepackage[format=plain,font=footnotesize,labelfont=bf,labelsep=period]{caption}
\usepackage{sidecap} 
\usepackage[export]{adjustbox}
\usepackage{subcaption}
\usepackage[font=small]{caption}
\usepackage{float}
\usepackage{comment}

\usepackage{amsmath} 
\usepackage{amssymb}  
\usepackage{amsthm}
\usepackage{mathtools}

\usepackage{xcolor}

\DeclareMathOperator{\Span}{span}
\DeclareMathOperator{\rank}{rank}
\DeclareMathOperator*{\argmin}{arg\,min}

\newtheorem{lemma}{Lemma}

\usepackage[normalem]{ulem}
\usepackage{paralist}	
\usepackage[space]{grffile} 
\usepackage{color}
\let\labelindent\relax
\usepackage{enumitem}
\usepackage{bm}
\usepackage{cancel}
\usepackage{hhline}
\usepackage[c2 , nocomma]{optidef}
\usepackage{oubraces}
\usepackage{textcomp}
\usepackage{booktabs} 
\usepackage{longtable}

\usepackage{moreverb, url}

\usepackage[ruled,vlined]{algorithm2e}
\usepackage{stackengine}
\usepackage{dirtytalk}

\usepackage{array}
\usepackage{verbatim}
\usepackage{upgreek}

\usepackage{amsfonts}
\usepackage{lipsum}
\usepackage{cite}
\usepackage{hyperref}

\IEEEoverridecommandlockouts
\begin{document}
\spacing{1}
\title{ \bf  Real-Time Learning of Predictive Dynamic Obstacle Models for Robotic Motion Planning}

\author{%
Stella Kombo$^{1}$, Masih Haseli$^{1}$, Skylar X. Wei$^{2}$, and Joel W. Burdick$^{1}$%
\thanks{*This work was supported by the Defense Advanced Research Projects Agency (DARPA) under the LINC program.}%
\thanks{$^{1}$S. Kombo, M. Haseli, and J. W. Burdick are with the Division of Engineering and Applied Science, California Institute of Technology, Pasadena, CA 91125, USA.
${\tt\small \{skombo, mhaseli, jburdick\}@caltech.edu}$}%
\thanks{$^{2}$S. X. Wei is with Applied Intuition, Mountain View, CA 94041, USA.
${\tt\small skylar@applied.com}$}%
}

\maketitle
\begin{abstract}
Autonomous systems often must predict the motions of nearby agents from partial and noisy data. This paper asks and answers the question: "can we learn, in real-time, a nonlinear predictive model of another agent's motions?" Our online framework denoises and forecasts such dynamics using a modified sliding-window Hankel Dynamic Mode Decomposition (Hankel-DMD). Partial noisy measurements are embedded into a Hankel matrix, while an associated Page matrix enables singular-value hard thresholding (SVHT) to denoise the Hankel matrix and estimate its rank. A Cadzow projection enforces structured low-rank consistency, yielding a denoised trajectory and local noise variance estimates. From this representation, a time-varying Hankel-DMD {\em lifted} linear predictor is constructed for multi-step forecasts. The residual analysis provides variance-tracking signals that can support downstream estimators and risk-aware planning. We validate the approach in simulation under Gaussian and heavy-tailed noise, and experimentally on a dynamic crane testbed. Results show that the method achieves stable variance-aware denoising and short-horizon prediction suitable for integration into real-time control frameworks.
\end{abstract}
\section{Introduction}\label{intro}

Autonomous robotic systems often operate in the presence of other dynamic and non-coordinated agents. For example, autonomous cars must navigate around vehicles, pedestrians, and cyclists \cite{gulzar2021survey,yaqoob2020autonomous}.
In autonomous drone racing, drones must avoid crashing into other drones on the racecourse \cite{spica2018realtime, hanover2023survey_arxiv} and in maritime robotics, a shipboard robotic arm or autonomous crane must plan for the payload's motion while compensating for the ship's sea-induced oscillations.
In such settings, the other agent's dynamics and intentions are typically unknown, while onboard sensing provides noisy, partial observations of the agent motions. Safe, efficient behavior therefore hinges on accurate short-horizon prediction of agent motion under uncertainty to enable collision-free planning and real-time control. Thus, this paper introduces a data-driven framework for real-time learning and short-horizon prediction from noisy, partial observations.

\subsubsection*{Related Work}\label{related_work}

Different approaches have been developed for integrating prediction and planning in dynamic environments where other agents’ dynamics and intentions are unknown. Classical geometric planners such as the Velocity-Obstacle (VO) family \cite{fiorini1998vo}, Reciprocal VO (RVO), Optimal Reciprocal Collision Avoidance (ORCA) and Acceleration Velocity Obstacles extensions (AVO)) \cite{vandenberg2011orca,vandenberg2008rvo,guy2009clearpath,vandenberg2011avo}, embed simple relative-motion estimates to predict collisions and select controls outside forbidden sets. These methods are efficient but assume simplified behavior models and near-perfect state estimates. However, real agents often violate these assumptions. For instance, some thrown objects may be roughly ballistic \cite{Goldstein2002Classical}, but aerodynamically complex bodies, e.g., frisbees, do not follow a simple ballistic model. Additionally, behavior-driven agents, e.g., pedestrians, require data-driven intent models to capture variability beyond constant-velocity assumptions \cite{Camara2020}. Lastly, the future trajectories of constraint-driven agents, e.g., autonomous vehicles, may be restricted by road geometry and traffic rules. Across all these examples, VO-style methods assume perfect knowledge of the moving agents’ states and neglect measurement noise or sensing delays \cite{hennes2012localization}. This reliance on idealized assumptions highlights the gap between geometric formulations and the stochastic, nonlinear realities of robotics applications.

When other agents’ dynamics are unknown, it is desirable to learn predictive models online from streaming observations. In practice, only partial and noisy measurements are available, while key latent states are unobserved, degrading forecasts and the safety of plans. Prior methods assume full-state observability for simplicity \cite{dutoit2010robotic}. This includes frameworks that learn dynamics directly from data \emph{offline}, e.g., classical AutoRegressive Moving Average (ARMA) models \cite{huang2015arma}, modern sequence learners such as Recurrent Neural Networks (RNNs), transformers, and tokenized representations \cite{talukder2024totem}. These off-line methods capture complex behavior but require large datasets and retraining \cite{saxena2025robot} and adapt poorly to distribution shift \cite{pierson2017deep}. \emph{Online} variants (e.g., Kalman Filtering (KF), recursive least squares, and online Gaussian processes \cite{ko2007gp}) update models during deployment, but presume known or structured noise distributions \cite{mckinnon2020context}.

Koopman theory is a complimentary alternative to learning these dynamic models. Formally, the Koopman operator models a nonlinear dynamical system via a \emph{linear operator} on a vector space of functions \cite{koopman1931hamiltonian}. Since the operator acts on an infinite dimensional vector space, real-time analysis and prediction is intractable. A practical approach is to \emph{approximate} the operator's action on a finite-dimensional subspace (which is typically of greater dimension than the state-space of the nonlinear system), enabling the use of well-developed methods from linear algebra and linear system theory. These approximations are often performed via orthogonal projections on the subspace of choice, also referred to as truncations of the operator's action. Dynamic Mode Decomposition (DMD) \cite{schmid2008dmd} and its variant Extended-DMD (EDMD) \cite{williams2015data} are well-known examples of such projection-based algorithms. Hankel-DMD uses time-delay embedding to estimate an effective state-space dimension \cite{arbabi2017koopman,kamb2020timedelay}. Such truncated Koopman-based models offer computationally feasible approximations of nonlinear dynamics with proven utility in robotics for prediction and control \cite{korda2018linear, folkestad2020episodic}. However, sensor measurements may be noisy, and since the fidelity of real-time Koopman predictors depends on data quality, noise degrades these learned models. 

Existing denoising frameworks to resolve noisy measurements prior to model fitting include classical low-rank denoising techniques i.e., truncated Singular Value Decomposition (tSVD), Principal Component Analysis (PCA) \cite{abdi2010pca} and Proper Orthogonal Decomposition (POD) \cite{berkooz1993pod}.  They provide optimal projections but rely on \emph{ad hoc} rank selection. Structured Hankel denoising \cite{yin2021lowrank} enforces temporal consistency but requires tuning and degrades under non-Gaussian noise. Scalable variants like randomized SVD \cite{halko2011rand} improve efficiency but depend on problem-specific thresholds, while Eigensystem Realization Algorithm (ERA) \cite{juang1985era} and Singular Spectrum Analysis (SSA) \cite{golyandina2001ssa} remain sensitive to noise.

\subsubsection*{Contribution}

We propose an adaptive, denoising, sliding-window Hankel–DMD framework for real-time prediction from noisy partial measurements. At each time step, a finite data buffer fits a local model, balancing responsiveness to non-stationary behavior with enough samples for reliability. We first denoise via Cadzow’s low-rank projection, with rank $\widehat r$ chosen by Singular Value Hard Thresholding (SVHT) on a Page matrix embedding.
We prove that, under mild conditions, Page and Hankel matrices have the same finite-sample rank, so the SVHT rank transfers to the Hankel matrix for Cadzow denoising. The resulting denoised Hankel matrices yield a sequence of local models and noise-variance estimates for uncertainty-aware planning. We validate the framework in both simulation and hardware experiments under Gaussian and non-Gaussian correlated heavy-tailed noise. We show that the framework offers robust denoised predictions suitable for real-time control and planning tasks.

\subsubsection*{Notation}
$\mathbb{R}$ and $\mathbb{N}$ denote real and natural numbers. $I_n$ and $0_{m\times n}$ denote the $n\times n$ identity and $m \times n$ zero matrices, respectively.  For matrix $A\in\mathbb{R}^{m\times n}$, $A^\top$, $A^\dagger$, and $\rank(A)$ denote its transpose, Moore–Penrose pseudo-inverse, and rank respectively.   We denote by $\operatorname{diag}(v)$, the $n\times n$ diagonal matrix with elements of $v\in \mathbb{R}^n$ on its main diagonal.   All linear combinations of vectors $\{v_1,\dots,v_k\}\subset \mathbb{R}^n$ is denoted by $\Span\{v_1,\ldots,v_k\}$.
We denote the variance of random variable $X$ by $\sigma_X^2$ and drop the subscript when the context is clear. For functions $f$ and $g$ with matching domain and co-domain, $f\circ g(x):= f(g(x))$ denotes their composition.

\section{Problem Definition} \label{sec:problem}

Consider a robot that operates in the vicinity of  another moving entity $\mathcal{O}$ whose behavior and dynamics are unknown. The robot's onboard sensors detect $\mathcal{O}$ and provide noisy, partial measurements of the obstacle state, denoted by $x_t \in \mathbb{R}^{n_x}$, which are sampled at a uniform interval $\Delta t \geq 0$, and capture kinematic quantities such as velocities or angular rates. While $x_t$ may not directly encode full Cartesian position, it carries sufficient temporal information to enable short-horizon motion prediction. The motion of $\mathcal{O}$ is governed by a discrete-time dynamical system with latent (unobserved) variables and unmeasured control inputs in the form: $z_{t+1} = \hat{f}(z_t,u_t)$ where $z_t$ denotes the state, which may include latent (unobserved) variables and $u_t$ an input from an unknown policy $u_t = c(z_t)$. Under these assumptions, we have an autonomous system 
\begin{align}\label{eq:system}
        z_{t+1} &= f(z_t) \,, 
        \nonumber\\
        x_{t}     &= C\,z_t + \eta_t,
\end{align}
where  measurement noise $\eta_t\in\mathbb{R}^{n_x}$ has an unknown distribution, and $C$ is the state-output map. This discrete time model generally arises from discretization of a continuous-time system. We maintain a sliding buffer of $N$ recent measurements $\mathcal{B}_t=\{x_{t-N+1},\ldots,x_t\}$ for online processing.

\textbf{Problem Statement:}
Given streaming noisy observations $\{x_t\}$ from a single trajectory and a sliding data buffer of fixed length $N$, construct, in real time, a tractable and adaptive local model. This model generates multi-step forecasts $\{\bar{x}_{t+1}, \ldots, \bar{x}_{t+N_h}\}$ over a prediction horizon $N_h \in \mathbb{N}$, suitable for downstream planning and control.

\section{Preliminaries}\label{sec:preliminaries}
Here, we define key concepts that are leveraged in our method such as Hankel and Page embeddings, Hankel-DMD, Cadzow  denoising, and Singular Value Hard Thresholding.

\subsection{Hankel and Page Matrices} 

Guided by Takens’ delay-embedding theorem (the dynamics of a system can be reconstructed from delay-coordinate maps of a single observable when the embedding dimension is sufficiently large \cite{Takens1981}), we apply delay embedding to the obstacle measurements to extract informative information about its dynamics. 
Given a sequence of $N$ measurements $\{x_{i-N+1},\ldots,x_i\}\subset\mathbb{R}^{n_x}$, we denote the associated  (block) \emph{Hankel} matrix $H^{\,L}_{i-N+1:i}\in \mathbb{R}^{(L n_x)\times (N-L+1)}$ by:
\begin{equation}\label{eq:hankel}
H^{\,L}_{i-N+1:i} \;=\;
\begin{bmatrix}
x_{i-N+1}   & x_{i-N+2}   & \cdots & x_{i-L+1} \\
x_{i-N+2}   & x_{i-N+3}   & \cdots & x_{i-L+2} \\
\vdots      & \vdots      & \ddots & \vdots    \\
x_{i-N+L}   & x_{i-N+L+1} & \cdots & x_{i}
\end{bmatrix}.
\end{equation}

The (block) \emph{Page}~\footnote{For $n_x = 1$, the matrices $H^{L}_{i-N+1:i}$ and $P^{L}_{i-N+1:i}$ take Hankel and Page forms, respectively. For $n_x > 1$, they become block Hankel and block Page. For simplicity, we refer to both cases as Hankel and Page structures.}matrix \(P^{\,L}_{i-N+1:i}\in\mathbb{R}^{(L n_x)\times m}\) partitions the same sequence into \(m\) non-overlapping blocks of length \(L\) (with \(m=N/L\))\footnote{Note that, unlike the Hankel matrix, in the Page matrix, the buffer length $N$ must be a multiple of the embedding window length $L$.}:

\begin{equation}\label{eq:page}
P^{\,L}_{i-N+1:i}=
\begin{bmatrix}
x_{i-N+1}   & x_{i-N+L+1} & \cdots & x_{i-L+1} \\
x_{i-N+2}   & x_{i-N+L+2} & \cdots & x_{i-L+2} \\
\vdots       & \vdots       & \ddots & \vdots    \\
x_{i-N+L}   & x_{i-N+2L}  & \cdots & x_{i}
\end{bmatrix}.
\end{equation}

\subsection{Hankel-DMD}\label{subsec:hankel-DMD}

We review a variant of Hankel-DMD (cf.~\cite{arbabi2017koopman,kamb2020timedelay}) suitable for our problem.  It uses delay embedding on the measurements of obstacle's observables ~\eqref{eq:system}. Given a measurement sequence $\{x_0, x_1, \ldots, x_n\}$, form the one step shifted Hankel matrices $H_{0:n-1}^L$ and $H_{1:n}^L$ according to the Eq. ~\eqref{eq:hankel}. Note that the columns of $H_{1:n}^L$ are shifted one time step forward from the columns of $H_{0:n-1}^L$.  To estimate the one-step dynamic prediction propagator, $A^{*}$, Hankel-DMD relies on the following least Frobenius-norm problem,  
\begin{equation}\label{eq:one_step_ls}
A^* \;=\; \argmin_{A}\,\|H_{1:n}^L - A H_{0:n-1}^L\|_F .
\end{equation}
The closed form solution for $A^*$ is
\begin{equation}\label{eq:one_step_closed_form-original}
A^* \;=\; H_{1:n}^L\,(H_{0:n-1}^L)^{\dagger} .
\end{equation}

Under standard assumptions of ergodicity and noise-free observables, Hankel-DMD can be used as a finite-dimensional approximation of the Koopman operator via delay embeddings \cite{arbabi2017koopman}. Our goal is to tackle the real-world cases where these assumptions do not hold.

\subsection{Cadzow Algorithm for Denoising}\label{sub:cadzow_proj}

Here, we review a variant of Cadzow's algorithm for denoising Hankel matrices\cite{Cadzow1988}. In the approximately linear {\em lifted} dynamics (e.g., Koopman-based approximations), noise-free data resides on a low-dimensional vector space. Hence, for a sufficiently large delay embedding window $L$, the noise-free Hankel trajectory matrix becomes low rank. However, with measurement noise, the observed Hankel matrix is generally full rank.   Let Hankel matrix $H^L \in \real^{(L\,n_x) \times (N-L+1)}$ be constructed from $N$ consecutive noisy measurements with embedding window $L$. Decompose $H^L$ as $H^L \;=\; \widehat H^L \;+\; \Delta$ with
   \begin{equation}\label{eq:cadzow_decomp}
     \operatorname{rank}\!\big(\widehat H^{\,L}\big)=r < \min\{L n_x,\,N{-}L{+}1\}
   \end{equation}
where $\Delta$ denotes measurement noise and $\widehat{H}^{\,L}$ is the noise-free Hankel matrix. The goal is to recover the (unknown) noise-free, low-rank Hankel matrix $\widehat H^{\,L}$ from $H^{\,L}$. Cadzow's algorithm aims for this objective by alternating projections of $H^L$ on the set of matrices with rank of at most $r$ with  projections of the result back onto the set of Hankel matrices.

Given a target rank $r$, define the projection map $\Pi_r$ onto the set of rank-$\le r$ matrices as
\begin{equation}\label{eq:Pi_r}
\Pi_r(H)\;:=\;\argmin_{\operatorname{rank}(Y)\le r}\,\|H-Y\|_F^2
\;=\;U\,\Sigma_r\,V^\top
\end{equation}
where $H=U\Sigma V^\top$ is the SVD  and $\Sigma_r$ is created by keeping the top $r$ singular values in $\Sigma$ while setting the rest to zero. The closed-form solution of~\eqref{eq:Pi_r} is a direct consequence of Eckart–Young–Mirsky theorem ~\cite{eckart1936approximation}.
Similarly, we define the projection map $\Pi_{\mathsf H}$ onto the set of Hankel matrices as
\begin{equation}\label{eq:Pi_Hankel}
\Pi_{\mathsf H}(M)\;:=\;\argmin_{Z\in\mathsf{Hankel}}\;\|M-Z\|_F^2.
\end{equation}
Interestingly, the optimization problem~\eqref{eq:Pi_Hankel} has a closed-form solution which can be calculated by replacing the members of each anti-diagonal of $M$ with their average. 
The procedure is as follows: decompose $M$ into $L$ by $N-L+1$ blocks similarly to the Hankel matrix\footnote{Note that each block is a column vector with $n_x$ rows similarly to Eq.~\eqref{eq:hankel}}. And denote by $M_{i,j}$ the $ij$th block. For each anti-diagonal offset $d=0,\ldots,N-1$, define:
\[
\mathcal I_d =\{(\ell,m):1\le \ell\le L,1\le m\le N{-}L{+}1,\ell+m-2=d\}\]
where $s_d \;=\; |\mathcal I_d|$ is number of elements in $\mathcal I_d$. Then, the anti-diagonal averaging can be computed as
\[
\big(\Pi_{\mathsf H}(M)\big)_{\ell,m}
\;=\;
\frac{1}{s_{\ell+m-2}}
\sum_{(i,j)\in \mathcal I_{\ell+m-2}} M_{i,j},
\]
for $1\le \ell\le L,\;1\le m\le N{-}L{+}1$.
The Cadzow algorithm successively applies the projections $\Pi_r$ and $\Pi_{\mathsf H}$ so that the sequence converges to a Hankel matrix, $\hat{H}^L$, that is approximately of rank $r$.
The solution of the Cadzow algorithm on $H^L$ after $n$ iterations is $[\Pi_{\mathsf H}\circ \Pi_r]^{\,n}\,(H^L)$ where $[\Pi_{\mathsf H}\circ \Pi_r]^{\,n}$ denotes the map created by composing $\Pi_{\mathsf H}\circ \Pi_r$, $n$-times with itself.

Notably, since projections $\Pi_r$ and $\Pi_{\mathsf H}$ can be computed in closed form, the Cadzow algorithm is an attractive choice for real-time denoising, as only a few iterations are sufficient to yield a reasonable signal-to-noise ratio.

\subsection{Singular Value Hard Thresholding (SVHT)}\label{subsec:rank}

Here we review a variant of Gavish and Donoho's result\cite{gavish2014optimal}. Consider a sequence (indexed by $a$) of noisy matrices \(Y_a=X_a+Z_a/\sqrt{n_a}\) where $Y_a\in\mathbb{R}^{m_a\times n_a}$ denotes a noisy matrix, $X_a$ the "true" low-rank matrix, and $Z_a$ denotes the noise matrix whose entries are i.i.d.\ zero-mean, unit-variance entries with finite fourth moment. As $a \to \infty$, the aspect ratio $m_a/n_a$ converge to $\beta \in (0,1]$\footnote{Following \cite{gavish2014optimal}, $a$ indexes a growing problem sequence with $Y_a\in\mathbb{R}^{m_a\times n_a}$, $m_a/n_a\to\beta$, and $Y_a=X_a+Z_a/\sqrt{n_a}$. Their AMSE–optimality results are asymptotic as $a\to\infty$.}. 

Given the SVD $Y_a = U\Sigma V^\top$ and a threshold $\tau$, we approximate the true matrix $X_a$ by $\widehat X_a \;=\; U\Sigma_{\geq \tau} V^\top$ where $\Sigma_{\geq \tau}$ is formed by zeroing the subthreshold elements of $\Sigma$.

Gavish and Donoho showed that the Asymptotic Mean Squared Error (AMSE)-optimal threshold for the data singular values depends only on the matrix aspect ratio $\beta$ via a constant $\lambda^\star(\beta)$:
\begin{equation}\label{eq:lambda_star}
\lambda^\star(\beta)
=\sqrt{\,2(\beta+1)+\frac{8\beta}{(\beta+1)+\sqrt{\beta^2+14\beta+1}}\,}.
\end{equation}
For the noise standard deviation $\sigma$, the optimal threshold is:
\begin{equation}
\tau^\star = \lambda^\star(\beta)\,\sigma\sqrt{n_a}.
\end{equation}
When $\sigma$ is unknown, the optimal threshold can be estimated using a data-driven approach through the Marchenko–Pastur (MP) law \cite{MarenkoPastur1967}, which gives the limiting eigenvalue distribution of \((1/n_a)\,{Z_a}{Z_a}^\top\) as \(m_a/n_a \to \beta\):
\[
p_\beta(\delta)=\frac{\sqrt{(\delta_+ - \delta)(\delta - \delta_-)}}{2\pi\beta\,\delta} \,
\mathbf{1}_{[\delta_-,\,\delta_+]}\!(\delta),\]
with \(p_\beta(\delta)=0\) outside \([\delta_-,\delta_+]\). The median of the MP distribution, denoted by $\mu_\beta$, depends only on $\beta$ and provides a normalization for estimating the noise level $\sigma$. Taking $\varsigma_{\mathrm{med}}$ as the median of the observed singular values, the data-driven AMSE--optimal threshold can thus be estimated as:
\begin{equation}\label{eq:data_driven_star_main}
\tau^\star(\beta) = \frac{\lambda^\star(\beta)}{\sqrt{\mu_\beta}}\,
\varsigma_{\mathrm{med}}.
\end{equation}
which remains asymptotically optimal under general white noise (i.i.d., zero mean, unit variance, finite fourth moment) with the same risk as in the Gaussian case \cite[Sec.~VI]{gavish2014optimal}.

\section{
Adaptive Hankel-DMD for Noisy Data
}\label{sec:method}
We now address the problem outlined in Sec.~\ref{sec:problem}. To enable real-time obstacle dynamics learning and motion prediction, we consider the use of Hankel-DMD (cf. Sec.~\ref{subsec:hankel-DMD}). However, applying Hankel-DMD in real-time robotics applications presents two key challenges:
\begin{itemize}
\item \textbf{Sensor noise:} sensor signals are typically noisy, which contaminates  matrices $H_{0:n-1}^L$ and $H_{1:n}^L$. In least-squares schemes such as~\eqref{eq:one_step_ls}, noise in both data matrices leads to biased or inconsistent estimators~\cite{gleser1981noisy}.
\item \textbf{Ergodicity and long trajectory data requirements:} as noted in Sec.~\ref{subsec:hankel-DMD}, one must assume ergodicity and long data sequences to  connect Hankel-DMD and the Koopman operator. Such assumptions rarely hold in robotics applications.
\end{itemize}

To tackle the noise issue, we introduce a Hankel matrix denoising scheme. We then propose an adaptive sliding-window variation of Hankel-DMD that continually updates the model. This approach ensures that the predictive model remains accurate as the object's state evolves over time.

\subsection{Denoising Hankel Matrices}

Consider obstacle output measurements $\mathcal{B}_t=\{x_{t-N+1},\ldots,x_t\}$ for buffer size $N$, collected up to time $t$. We construct a Hankel matrix $H^{\,L}_{t-N+1:t}$  from $\mathcal{B}$ according to  Eq.~\eqref{eq:hankel}. Since the measurements are noisy, we seek to denoise the Hankel matrix $H^{\,L}_{t-N+1:t}$ via  Cadzow's algorithm, as described in Sec.~\ref{sub:cadzow_proj}.  A key step in this algorithm is the low-rank projection $\Pi_r$ in Eq.~\eqref{eq:Pi_r}, which requires knowledge of the effective rank $r$ of the noise-free Hankel matrix. In practice, this typically unknown rank must be estimated. An rank estimat error can cause over-smoothing (if underestimated) or noise amplification (if overestimated). To avoid these problems, we adopt a principled, data-driven rank selection strategy based on Singular Value Hard Thresholding (SVHT) (Sec.~\ref{subsec:rank}). However, the SVHT framework assumes i.i.d. noise, an assumption that is violated in Hankel matrices due to repeated entries. To address this, we employ Page matrices (Eq.~\eqref{eq:page}), whose partition of the measurement buffer into non-overlapping blocks avoids correlation.

\subsubsection{Step I: Page-Hankel Rank Transfer}
As a first step toward estimating the rank of the noise-free Hankel matrix using SVHT on Page matrices, we show that, under mild conditions, the Page and Hankel embeddings of the \emph{same} noise-free measurement buffer share the same rank.

\begin{lemma}[Page–Hankel Rank Equivalence]\label{lem:rank-equivalence}
Consider a locally valid linear output model
\[
    z_{k+1}=A z_k,\qquad x_k = C z_k,
\]
with $z_k\in\mathbb{R}^{n_z}$, $A\in\mathbb{R}^{n_z\times n_z}$, and $C\in\mathbb{R}^{n_x\times n_z}$. Let  $\{x_{0} \dots,x_{N-1}\}\subset\mathbb{R}^{n_x}$ 
be the noise-free measurements generated by the system above from the initial condition $z_0$. Let $L \geq n_z$
be the embedding window and let $N = dL$, with $d \geq L$. Define\footnote{To avoid confusion with superscript $L$, we use  $(A)^L$ to denote the matrix created by raising $A$ to the power of $L$.} $B = (A)^L$ and let $P^L$ and $H^L$ be the block Page and Hankel matrices constructed via the data. Then, $\rank(P^L) = \rank(H^L)$ if 
\begin{equation*}
    \Span\{z_0, Bz_0, \ldots, (B)^{d-1} z_0\} = \mathbb{R}^{n_z}.
\end{equation*}
\end{lemma}
\begin{proof}[Proof of Lemma~\ref{lem:rank-equivalence}]
Given the definition of system and the initial condition, one can write $x_i= C(A)^i z_0$ for $i \in \{0,\ldots,N-1\}$. Therefore, one can easily decompose the Page and Hankel matrices as
\begin{equation}\label{eq:page-hankel-decompose}
P^L=EF, 
\qquad
H^L=EG,
\end{equation}
where 
\begin{align*}
E &= [C^\top, A^\top C^\top, \ldots, \big((A)^{L-1}\big)^\top C^\top]^\top \in \mathbb{R}^{(n_x L)\times n_z},
\\
F&=[z_0,\,Bz_0,\,(B)^2 z_0,\,\ldots,\,(B)^{d-1}z_0] \in\mathbb{R}^{n_z\times d},
\\
G&=[z_0,\,Az_0,\,(A)^2z_0,\,\ldots,\,(A)^{(d-1)L}z_0] \in\mathbb{R}^{n_z\times((d-1)L+1)}.
\end{align*}
Note that by the hypothesis, $\rank(F)=n_z$, which implies that $F$ has full row rank. Moreover, noting that $B=(A)^L$, one can easily see that all columns of $F$ are in the set of columns of $G$; hence, $\rank(G) \geq \rank(F) = n_z$. However, note that $G \in\mathbb{R}^{n_z\times((d-1)L+1)}$ has more columns than rows; therefore, its rank can be at most $n_z$. Hence, one can conclude that both $F$ and $G$ have full row rank and
\begin{equation}\label{eq:rankFG}
\rank(F) = \rank(G) = n_z.
\end{equation}
Moreover, using~\eqref{eq:page-hankel-decompose}, the rank equality~\eqref{eq:rankFG}, and noting that $\rank(E) \leq n_z$ (due to its size) one can write
\begin{align}\label{eq:rank-upper-bound}
\rank(P^L) \leq \min(\rank(E),\rank(F)) = \rank(E),
\nonumber \\
\rank(H^L) \leq \min(\rank(E),\rank(G)) = \rank(E).
\end{align}
On the other hand, Sylvester’s rank inequality gives
\begin{align}\label{eq:rank-lower-bound}
\rank(P^L) = \rank (EF) \geq \rank(E)+\rank(F)-n_z,
\nonumber \\
\rank(H^L) = \rank (EG) \geq \rank(E)+\rank(G)-n_z.
\end{align}
Inequalities~\eqref{eq:rank-upper-bound}-\eqref{eq:rank-lower-bound} in conjunction with~\eqref{eq:rankFG} lead to $\rank(P^L) = \rank(H^L) = \rank(E)$ concluding the proof.
\end{proof}

Lemma~\ref{lem:rank-equivalence} provides a convenient way to connect the rank of Page and Hankel embedding of noise-free measurement buffers. Before proceeding to rank estimation, we briefly explain that the conditions of Lemma~\ref{lem:rank-equivalence} are mild. First, it is worth noting that with a sufficiently high sampling frequency, the system states associated with a given measurement buffer cluster within a small region of the state space, where a local linear model can accurately approximate the system’s behavior\footnote{This suggests that the buffer window size and sampling frequency should be chosen jointly to balance fidelity and sensitivity. An excessively large buffer reduces sensitivity to low-rank structures that capture local behavior, whereas an overly small buffer risks discarding expressive information about the system dynamics.}. In addition, the condition $\Span\{z_0, Bz_0, \ldots, (B)^{d-1} z_0\} = \mathbb{R}^{n_z}$ is generic\footnote{If the elements of $A$ (with $B = (A)^L$) and $z_0$ in Lemma~\ref{lem:rank-equivalence} are drawn from a reasonable distribution, the condition holds with probability one.} and holds for almost all matrices $B$ and vectors $z_0$.

Now, with Lemma~\ref{lem:rank-equivalence} at our disposal, we can estimate the rank of the noise-free Hankel matrix associated with $H^{\,L}_{t-N+1:t}$ by applying SVHT on the Page matrix constructed with the same data sequence. 
\subsubsection{Rank Estimation using SVHT}
To apply the SVHT (cf. Sec.~\ref{subsec:rank}) on the Page matrix, the number of rows in the matrix should be less than or equal to the number of columns. Therefore, for the rest of the paper, we consider the buffer length $N = m L n_x$, with $m \geq L n_x$. Note that this choice already satisfies the size condition in Lemma~\ref{lem:rank-equivalence}.

To construct the Page matrix $P^{\,L}_{t-N+1:t} \in \real^{L n_x \times m}$ from the same data sequence used in $H^{\,L}_{t-N+1:t}$.
By applying the SVHT formula in Eq.\eqref{eq:data_driven_star_main}, on the Page matrix, we can estimate the singular value threshold $\tau^*$ and approximate the effective rank of $P^{\,L}_{t-N+1:t}$ as
\begin{equation}\label{eq:svht_rank_hankel}
\widehat r \;=\; \#\bigl\{\, i\in\{1,\ldots,L n_x\}:\ \varsigma_i \ge \tau^\star \,\bigr\}
\end{equation}
where $\{\varsigma_i\}_{i=1}^{L n_x}$ are the singular values of $P^{\,L}_{t-N+1:t}$. 

It is important to note that the rank $\widehat r$ estimated via Page SVHT is not fixed but evolves as the sliding window advances. As the buffer shifts, the local data distribution changes, and the singular spectrum captures the instantaneous richness of the underlying dynamics. Consequently, small fluctuations or occasional variations in $\widehat r$ are to be expected. Far from being a drawback, this adaptivity ensures that the thresholding remains sensitive to regime shifts and transient behaviors, ultimately producing more reliable models when the system exhibits greater dynamical complexity.

Sec.~\ref{subsec:rank} also yields a conservative local noise–variance estimate. Singular values below the cutoff in Eq.~\eqref{eq:data_driven_star_main} are treated as noise, while those above define the rank via Eq.~\eqref{eq:svht_rank_hankel}. Let \(\varsigma_{\mathrm{med}}\) be the median singular value of the \emph{Page} matrix with aspect ratio \(\beta\) and \(m\) columns, and let \(\mu_\beta\) denote the median of the Marchenko–Pastur law \cite{MarenkoPastur1967} at aspect ratio \(\beta\). Then the noise variance $\widehat{\sigma}^{2}$ can be conservatively estimated as $\widehat{\sigma}^{2} \approx \varsigma_{\mathrm{med}}^{2}/\mu_{\beta}\,m$. 
Under the assumption of i.i.d. sensor noise, we approximate the delay–state noise covariance as \(\widehat{\nu}\approx \widehat{\sigma}^{2} I_{L n_x}\). This approximation treats each delay state as being perturbed by independent, isotropic noise, which is appropriate in the i.i.d. setting but may underestimate correlations in more structured noise processes. A concrete example of leveraging such variance information within a risk-aware framework is provided in \cite{swei}.

\subsubsection{Cadzow Algorithm}
Now, by using the rank estimate $\hat{r}$ in Eq.\eqref{eq:svht_rank_hankel} and utilizing Lemma~\ref{lem:rank-equivalence}, we can finally apply the Cadzow algorithm in Sec.\ref{sub:cadzow_proj} on Hankel matrix $H^{\,L}_{t-N+1:t}$ with effective rank $\hat{r}$ until we approximate a denoised Hankel matrix which we denote by $\widehat H^{\,L}_{t-N+1:t}$.

\subsection{Online Identification and Multi-Step Prediction}
With the denoised Hankel matrix $\hat H^{\,L}_{t-N+1:t}$ obtained, we proceed to model the system dynamics using Hankel-DMD.
However, as mentioned earlier, to connect the behavior of the dynamics learned by Hankel-DMD in Eqs.\eqref{eq:one_step_ls}-\eqref{eq:one_step_closed_form-original} to the system's global behavior via Koopman operator, one often requires assumptions of ergodicity on the system and access to very long trajectories. These assumptions rarely hold in robotics applications. To circumvent this issue, we apply an adaptive sliding window strategy to build a linear model for each buffer and update it as we receive new data. 

Formally, let $\hat H^{\,L}_{t-N+1:t-1}$ and $\hat H^{\,L}_{t-N+2:t}$ be the Hankel matrices constructed by taking the first and last $N-L$ columns of $\hat H^{\,L}_{t-N+1:t}$ respectively. 

We then solve a least Frobenius-norm problem similar to Eq.~\eqref{eq:one_step_ls} as:
\begin{equation}\label{eq:one_step_pred_ls}
\widehat A_t = \arg\min_{A_t}\,\big||\hat {H}_{{t-N+2:t}}^{L}-A_t\hat{H}_{t-N+1:t-1}^{L}\big||_F
\end{equation}
with the closed-form solution 
\begin{equation}\label{eq:one_step_closed_form}
\widehat A_t\;=\; \hat{H}_{{t-N+2:t}}^{L}\,(\hat{H}_{t-N+1:t-1}^L)^{\dagger}.
\end{equation}
We then define the local predictor dynamics as
\begin{equation}\label{eq:linear-predictor}
  \psi_{m+1} = \widehat{A}_t\, \psi_m, \quad m \in \{t, t+1, \ldots\}.
\end{equation}
If we initialize the system~\eqref{eq:linear-predictor} by setting $\psi_t$ to be equal to the last column of $\hat H^{\,L}_{t-N+2:t}$ (which corresponds to the denoised version of the lifted state at the current time), we can forecast the future outputs of system~\eqref{eq:system} by running the predictor~\eqref{eq:linear-predictor} and extracting the last block of the delay embedded state as
\begin{equation}\label{eq:multi_step_rollout}
\bar{x}_{t+j} \;=\; D \, \big(\widehat{A}_t\big)^{j}\,\psi_t^* \quad j \in {1, \ldots, N_h},
\end{equation}
where $\psi_t^*$ is set to be the last column of $\hat H^{\,L}_{t-N+2:t}$ and matrix $D \in \real^{n_x \times L n_x}$ is defined as $D = [0_{n_x \times (L-1)n_x}, I_{n_x}]$.

At the next time step, we advance the buffer to $\mathcal{B}_{t+1}=\{x_{t-N+2},\ldots,x_{t+1}\}$ and repeat the aforementioned procedure to produce a time-varying sequence of linear predictors characterized by $\{\widehat A_t\}_{t\ge N}$. These predictors are well suited for Model Predictive Control (MPC) by providing linear predictions over the planning horizon.
Algorithm \ref{alg:page_svht_cadzow} compiles the full sliding-window procedure.

\begin{algorithm}[htb]
\caption{Adaptive Sliding-Window Page--Hankel DMD Predictor}
\label{alg:page_svht_cadzow}
\SetKwInput{KwIn}{Input}
\SetKwInput{KwOut}{Output}
\KwIn{Embedding window $L$ 
\newline
Buffer Length $N = m L n_x$ with $m\geq L n_x$
\newline
Buffer $\mathcal{B}_t=\{x_{t-N+1},\ldots,x_t\}$ with $x_t\in\mathbb{R}^{n_x}$
\newline
Prediction horizon $N_h$}
\KwOut{Denoised Hankel $\widehat H^{\,L}_{t-N+1:t}$
\newline Model $\widehat A_t$ and Forecasts $\{\bar {x}_{t+1},\ldots,\bar {x}_{t+N_h}\}$
\newline Noise covariance estimate $\widehat\nu$}
\BlankLine

\textbf{(1) Page \& SVHT:}
Form $P^{\,L}_{t-N+1:t}$ per Eq.~\eqref{eq:page} 
\newline
Compute SVD $P^{\,L}_{t-N+1:t}=U\,\mathrm{diag}(\varsigma_i)\,V^\top$.
\newline
Set $\beta=(L n_x)/m$ and $\varsigma_{\mathrm{med}}=\mathrm{median}\{\varsigma_i\}$
\newline
Set cutoff $\tau^\star =\lambda^\star(\beta)/{\sqrt{\mu_\beta}}$ via Eq.~\eqref{eq:data_driven_star_main}
\quad $\widehat r \leftarrow \#\{i:\,\varsigma_i\ge\tau^\star\}$ per Eq.~\eqref{eq:svht_rank_hankel}) \smallskip
\newline
\textbf{(2) Cadzow Algorithm for $J$ Iterations:} $\widehat H^{\,L}_{t-N+1:t} \gets [\Pi_{\mathsf H}\!\circ\!\Pi_{r}]^{J}\!\big(H^{\,L}_{t-N+1:t}\big)$
\smallskip

\textbf{(3) Hankel Partitions and Predictor at $t$:}
\newline
$\hat{H}_{t-N+1:t-1}^L, \hat{H}_{t-N+2:t}^L$
\newline
$\widehat A_t\;=\; {\widehat{H}}_{{t-N+2:t}}^{L}\,(\widehat{H}_{t-N+1:t-1}^L)^{\dagger}$ as per Eq.~\eqref{eq:one_step_closed_form}
\smallskip

\textbf{(4) Prediction ($N_h$ steps):}
Let ${\psi_t}^\star$ be the last column of $\widehat H^{\,L}_{t-N+2:t}$. For $j=1,2,\ldots,N_h$, $\bar{x}_{t+j}=D(\widehat A_t)^j\,{{\psi}_t}^\star$ per Eq.~\eqref{eq:multi_step_rollout} \smallskip

\textbf{(5) Noise Variance Estimation:}
With $m$ Page columns and MP median $\mu_\beta$, set
$\widehat\sigma^{2}=\varsigma_{\mathrm{med}}^{2}/(\mu_\beta\,m)$ and
$\widehat\nu=\widehat\sigma^{2} I_{L n_x}$. \smallskip

\textbf{(6) Slide Window:}
$\mathcal{B}_{t+1}\leftarrow\{x_{t-N+2},\ldots,x_{t+1}\}$ and repeat (1)–(5).
\end{algorithm}

\subsection{Comparison with Alternative Online Prediction Methods} \label{sec:comparison}
Classical online noise filtering approaches, such as the KF and its nonlinear variants\cite{Maybeck1979_kf}, assume parametric dynamics and structured noise models. While online Gaussian processes\cite{soton259182, var_KF} and neural sequence models\cite{HS-neural_seq, GP_2015} provide strong predictive capability, they can require significant offline training and careful tuning for online deployment. In contrast, DMD\cite{Zhang2017onlinedmd} and EDMD\cite{williams2015data} are operator-theoretic system identification methods that approximate the Koopman operator from data. These formulations do not account for measurement noise and can exhibit bias under finite, noisy datasets\cite{hemati2017debiased, Dawson2016}, particularly when nonlinear lifting is used\cite{Haseli2019NoiseResilientEDMD}. Although noise-aware formulations and regularized variants exist\cite{Masih_robolearning}, these methods mainly improve dynamical consistency rather than provide a framework for denoising streaming measurements in finite-data regimes. our proposed method explicitly performs structured low-rank denoising prior to short-horizon prediction. Designed for finite, streaming data with limited computational budgets, the framework combines adaptive rank selection with structured matrix embeddings to enable stable real-time performance.

\section{Simulations and Experiments}
\label{sec:exp}
We apply Algorithm~\ref{alg:page_svht_cadzow} to validate  out framework in simulations and real-world experiments. Our evaluation examines (i) signal–to-noise (SNR) separability, (ii) noise variance estimation, and (iii) multi-step prediction accuracy. Experiments consider open-loop, short-horizon prediction to assess the accuracy, stability and noise robustness of the learned dynamics model. These properties are critical for downstream planning and control frameworks \cite{Das2026SafePayloadCrane}.

\subsection{Simulation: Noisy Unicycle}\label{sec:sim-uni}
Consider a reference unicycle model
\(
  \dot x = u_1\cos\theta,
  \dot y = u_1\sin\theta,
  \dot\theta = u_2
\),\cite{klancar2017wheeled,aicardi1995unicycle}, which serves as a dynamic obstacle moving along a figure-eight trajectory of amplitude $a=3\,\mathrm{m}$ and period $T=40\,\mathrm{s}$. An ego agent seeks to reach its goal while avoiding collisions with the unicycle, but has access to only noisy velocity measurements of the unicycle rather than the full state. Sampling the reference trajectory at
$\Delta t=0.02\,\mathrm{s}$ yields the forward velocity profile $u_1^{\star}(t)$. We inject sensor noise from two distributions: (i) i.i.d.\ Gaussian, $\eta_k \sim \mathcal{N}(0,\sigma_x^2)$ with $\sigma_x=0.25\,\mathrm{m/s}$ and (ii) correlated heavy-tailed AR(1)–Laplace, $\eta_t = \rho\,\eta_{t-1} + w_t,\,
w_t \sim \mathrm{Laplace}(0,b),\; |\rho|<1$. For fair comparison, $b$ is chosen so that the stationary variance matches the Gaussian case, $\operatorname{Var}[\eta_t] = \sigma_x^2$, i.e.\ $b=\sqrt{(1 -\rho^2)\sigma_x^2/2}$. In both cases, the noisy measurement stream is $v_t = u_1^{\star}(t) + \eta_t$.

\begin{figure}[tb]
    \centering
    \begin{minipage}{\linewidth}
        \centering
        \includegraphics[width=0.95\linewidth]{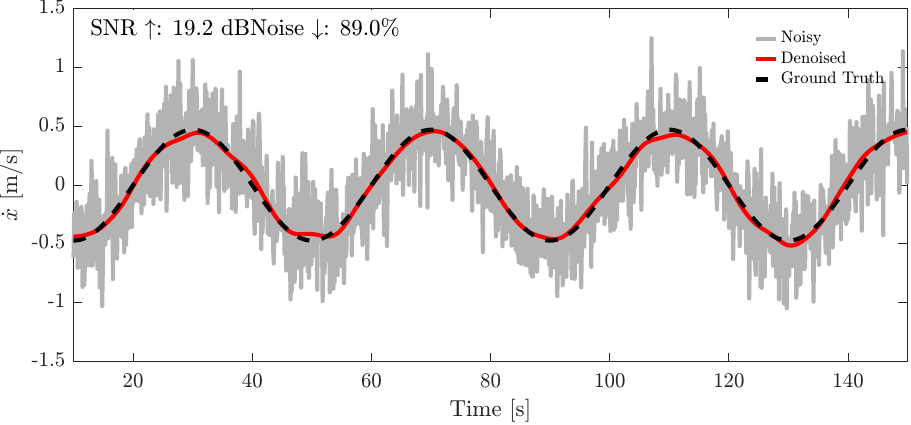}
        \begin{minipage}{0.49\linewidth}
            \centering
            \includegraphics[width=.95\linewidth]{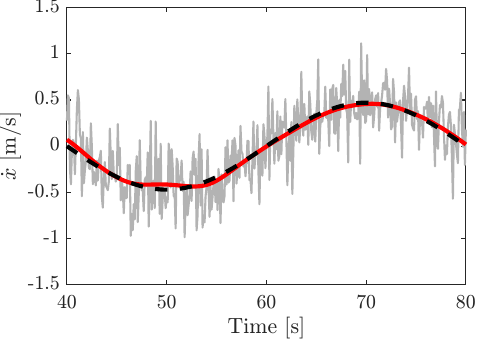}
        \end{minipage}
        \hfill
        \begin{minipage}{0.49\linewidth}
            \centering
            \includegraphics[width=.95\linewidth]{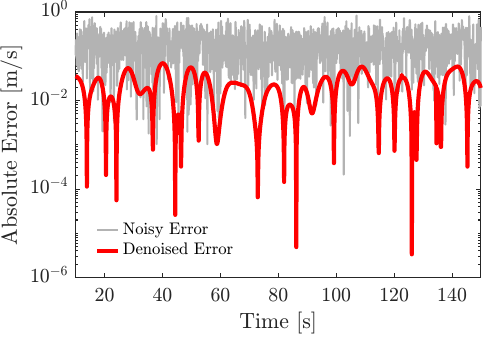}
        \end{minipage}
        \caption{\scriptsize Gaussian noise: Top: noisy (gray), denoised (red) and ground truth (black dashed); call-outs show SNR improvement and average noise reduction. Bottom: zoomed segment and log-scale error.}
        \label{fig:denoise-gauss}
    \end{minipage}
\end{figure}

\textbf{Denoising results (Gaussian)}: Fig.~\ref{fig:denoise-gauss} demonstrates that our Page-Hankel SVHT framework achieves stable recovery of the ground-truth velocity profile under Gaussian sensor noise. The denoised estimate (red, top panel) tracks the ground truth (dashed black) without visible phase lag while suppressing high-frequency artifacts that often contaminate raw measurements. Over the $160$\,s trajectory, the method delivered an SNR gain of $19.2$\,dB and an average noise reduction of $89.0\%$ (call-out, top-left), corresponding to nearly an order-of-magnitude improvement over baseline. The zoomed view (bottom-left) highlights that turning points and low-curvature segments, often degraded by conventional low-pass filters, were preserved, while the log-scale absolute error confirms a $10$-$100\times$ reduction in residual magnitude. These results indicate that the method can denoise aggressively without compromising structural features critical for downstream control. 
\begin{figure}[tb]
    \centering
    \begin{minipage}{\linewidth}
        \centering
        \includegraphics[width=\linewidth]{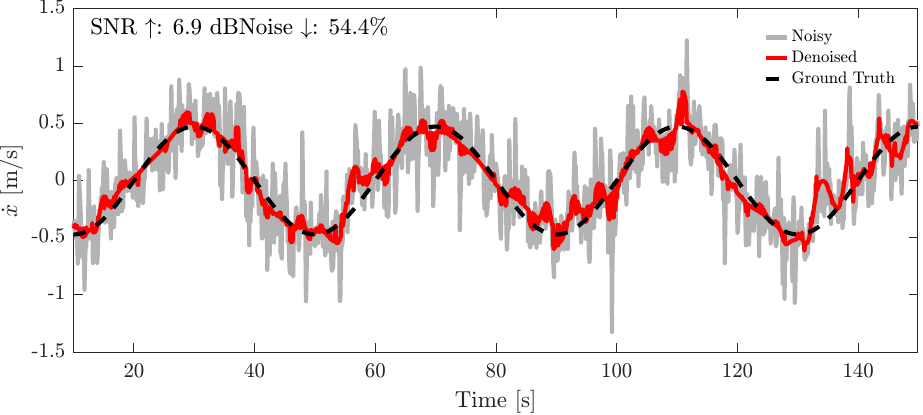}
        \begin{minipage}{0.49\linewidth}
            \centering
            \includegraphics[width=\linewidth]{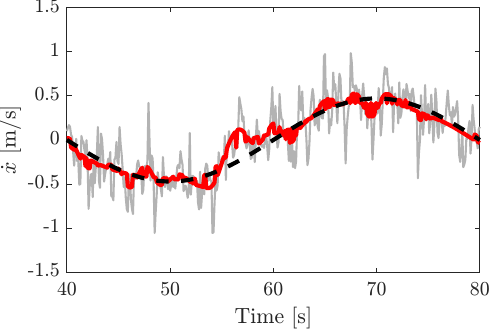}
        \end{minipage}
        \hfill
        \begin{minipage}{0.49\linewidth}
            \centering
            \includegraphics[width=\linewidth]{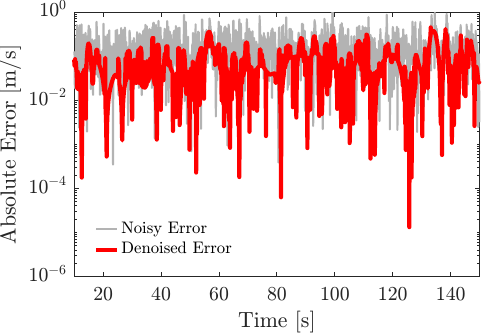}
        \end{minipage}
        \caption{\scriptsize AR(1)-Laplace noise: The denoising algorithm remains effective under heavy-tailed, correlated disturbances: call-outs: SNR $+6.9$\,dB and $54.4\%$ noise reduction.}
        \label{fig:denoise-laplace}
    \end{minipage}
\end{figure}

\textbf{Denoising results (Heavy tailed, correlated noise):} To probe robustness under non-Gaussian disturbances, we injected temporally correlated AR(1)-Laplace noise (Fig.~\ref{fig:denoise-laplace}). Despite heavy tails and temporal correlation, we observed an SNR gain of $6.9$\,dB and an average noise reduction of $54.4\%$, with no visible phase lag. While performance is lower than under Gaussian noise, as to be expected, the degradation remains gradual. Crucially, the low-rank Hankel structure proved distribution-agnostic i.e., SVHT adaptively trims singular values inflated by heavy-tailed noise requiring explicit Gaussian assumptions on the noise statistics.

\textbf{Comparison with Extended Kalman Filter (EKF)-based denoising:} For reference, we evaluated an EKF as a denoising pipeline under the same Gaussian disturbance conditions defined in Sec.~\ref{sec:sim-uni}. When the process and measurement covariance matrices are carefully tuned using approximate knowledge of the noise moments, the EKF achieves moderate reconstruction quality, yielding an SNR of approximately $0.6$\,dB. However, even in this tuned setting the EKF exhibits a noticeable phase delay (Fig.~\ref{fig:denoise-ekf}, top), which increases the point-wise reconstruction error used in the SNR calculation and may
be problematic for downstream control pipelines that require temporally aligned state estimates for stable feedback and prediction. In contrast, the proposed Page/Hankel-based method achieves an SNR of $19.2$\,dB under the same disturbance conditions.

In practical robotic systems, the true measurement noise statistics are typically unknown and may vary over time. Since the EKF relies on explicit specification of these noise moments to compute the Kalman gain, its performance is sensitive to covariance mismatch. Under mild mismatch in the assumed noise statistics, reconstruction quality yields an SNR of approximately $10.0$\,dB (Fig.~\ref{fig:denoise-ekf}, bottom). Our proposed Page/Hankel framework estimates the noise structure directly from the data through low-rank structure and adaptive SVHT thresholding, avoiding explicit noise variance specification while maintaining consistent reconstruction performance across noise regimes.

\begin{figure}[!tb]
    \centering
    \begin{minipage}{\linewidth}
        \centering
        \includegraphics[width=\linewidth]{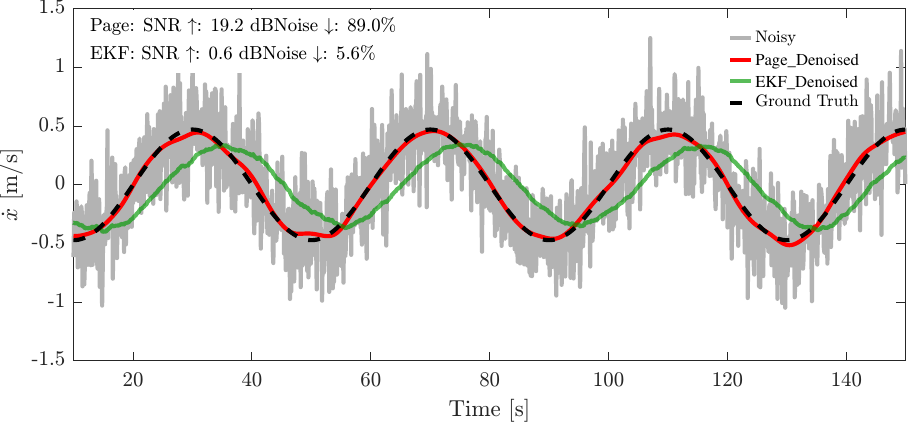}
        \vspace{0.6em}
        \includegraphics[width=\linewidth]{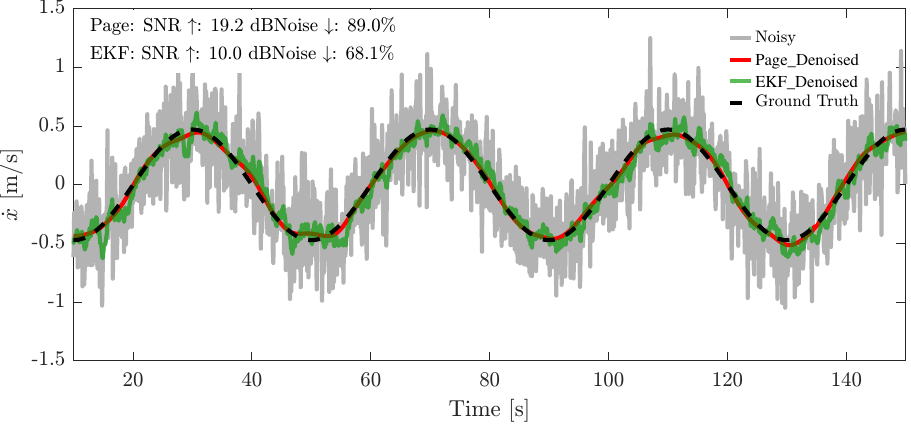}
        \caption{\scriptsize Comparison of Page/Hankel denoising and EKF filtering under temporally correlated Gaussian noise (AR(1)). Top: EKF with tuned covariance matrices (call-outs: SNR $+0.6$\,dB). Bottom: EKF under covariance mismatch (cal-outs: $+10.0$\,dB).}
        \label{fig:denoise-ekf}
    \end{minipage}
\end{figure}

\subsection{System Hardware and Experiments}\label{sec:stewart}

This section experimentally validates the complete method in a setup (Fig.~\ref{fig:stewart-structure}) motivated by a naval application. Currently, payloads are loaded into their silos by a ship-board crane while the ship docks in a port, where no ocean swells affect the safety of the loading process. It would be desirable to load these payloads while the ship is underway, potentially in heavy seas. In active sea states, the ship sways in response to ocean swells. Anticipating deck motion via a learned predictive model would enable the crane controller to compensate during payload delivery.

The testbed is a scaled crane mounted on an \emph{E2M eMove eM6-300-1500} five-DoF Stewart platform that can reproduce wave-induced ship deck motions. A VN-100 IMU measures orientation (quaternion) and angular velocity of the simulated moving base at 30\,Hz. These measurements correspond to the platform (deck) motion only, which is modeled as a time-varying base motion input to the crane-payload system. Using these IMU data, the framework generates a $31$-step horizon open-loop forecast ($\approx$1.0\,s) of the platform motion. Payload and crane dynamics are not directly predicted and are instead assumed to evolve downstream in response to the predicted base motion.
\begin{figure}[!tb] 
    \centering \includegraphics[width=1.0\linewidth]{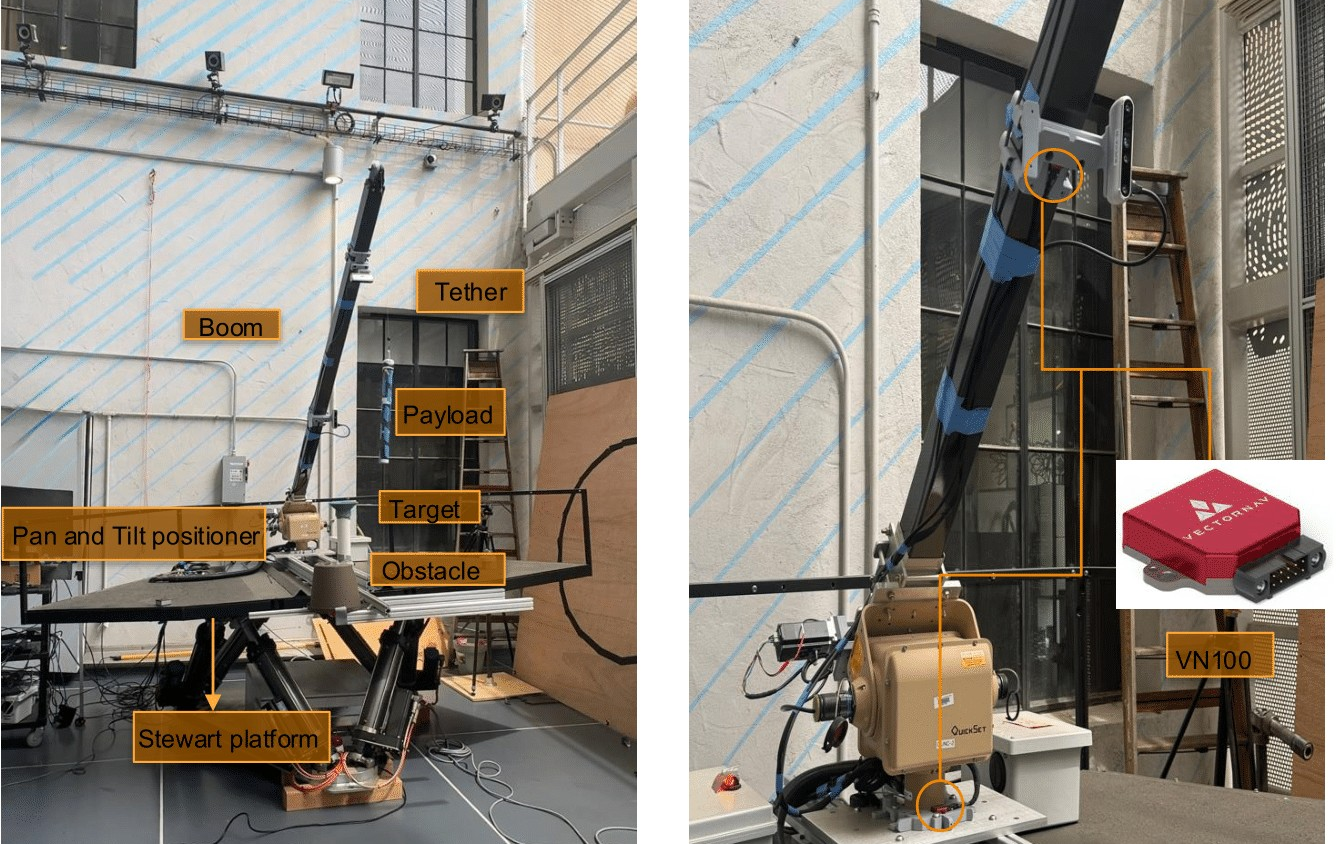}
    \caption{Stewart‐platform testbed. \textbf{Left:} Moving‐base crane with target, obstacle, payload, and arm/tip cameras. \textbf{Right:} Base-mounted VectorNav VN-100 IMU supplying orientation and angular rates to the sliding window Hankel-DMD predictor.}
    \label{fig:stewart-structure}
\end{figure}
This setup enables two evaluations: (i) whether the variance-stable denoising observed in simulation persists under real IMU noise and (ii) whether short-horizon sliding window Hankel-DMD predictions remain within a bounded threshold suitable for MPC integration. Fig.~\ref{fig:stewart-prediction} (top) shows a representative trajectory where $N=250$ sample context buffer (gray) feeds the pipeline, generating $31$-step forecasts (red) aligned with the ground truth (dashed black). The open-loop predictor achieves an RMSE of $0.012$\,m/s. 

Given a fixed error-tolerance $\varepsilon=0.04$\footnote{The tolerance $\varepsilon$ is chosen based on prior operational constraints of the experimental platform. In earlier implementations using a baseline constant-velocity predictor, prediction errors up to $0.08$\,m were acceptable for reliable payload insertion. We therefore set $\varepsilon = 0.04$ as a conservative threshold that remains consistent with the hardware constraints.}, we define a violation-duration metric on the prediction error $e_t$ as
\[
  J_t \;=\; \Delta t \sum_{t=1}^{N_h} \mathbf{1}\!\big(e_t \ge \varepsilon\big), \quad \text{with} \quad e_t := \lVert\bar{x}_t - x_t\rVert_2
\]
where $N_h$ defines the prediction horizon, $\bar{x}_t$ the prediction forecast, $x_t$ the measured ground truth, $\Delta t$ the sample period, and $\mathbf{1}(\cdot)$ the indicator (1 if the condition holds, 0 otherwise). In Fig.~\ref{fig:stewart-prediction} (bottom), $J_t=89.0$\,s ($1.6\%$ of horizon\footnote{The fraction of the horizon spent above threshold,
$ \%\text{violating} \;=\; 100 \times \frac{J_t}{T_{\mathrm{hor}}}, 
  \qquad T_{\mathrm{hor}} := N_h\,\Delta t.$}), certifying that prediction errors are bounded within the tolerance for $98.4\%$ of the time.

Further investigation of Fig.~\ref{fig:stewart-prediction} (top) reveals occasional transient spikes in the predicted trajectory, after which the forecasts quickly return to a nominal accuracy. These anomalies arise when the local embedding window captures a regime shift or abrupt disturbance, briefly mis-aligning the rank estimate. Importantly, such events are short-lived and self-correct as the sliding buffer refreshes, with the subsequent predictions re-stabilizing around the future measurements. This highlights both the adaptivity and the finite-sample sensitivity of the method, i.e., while transient outliers can occur, the framework consistently recovers without persistent drift.

\begin{figure}[tb]
\vskip -0.1 true in
    \centering
    \includegraphics[width=0.70\linewidth]{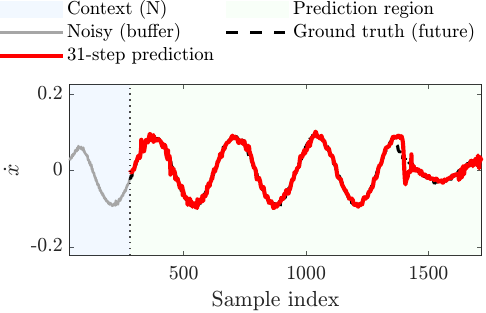}
    \vspace{1ex}
    \includegraphics[width=0.70\linewidth]{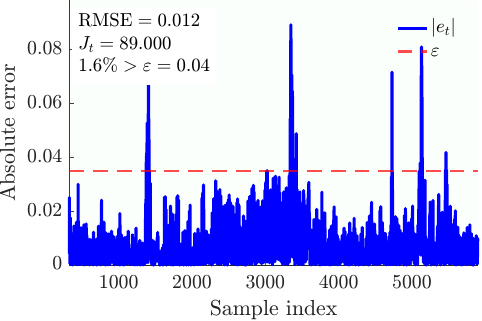}
  \vskip -0.08 true in
    \caption{\footnotesize Top: context buffer ($N=250$, gray) and 31-step forecasts (red) overlaid with ground truth (black dashed). Bottom: absolute prediction error $|e_t|$ (blue) compared to threshold $\varepsilon$ (red dashed).}
    \label{fig:stewart-prediction}
\end{figure}

Fig.~\ref{fig:stewart-spectra} illustrates the evolution of the eigenvalues of the learned predictors across sliding windows. As the system behavior changes along the trajectory, the models adapt accordingly, and the eigenvalues shift smoothly while consistently remaining within the unit circle, ensuring Schur stability.
The gradual evolution of the spectrum highlights the stability and time-varying nature of the local Hankel-DMD models throughout the experiment.

\begin{figure}[htb]
    \centering
    \includegraphics[height=3.5cm]{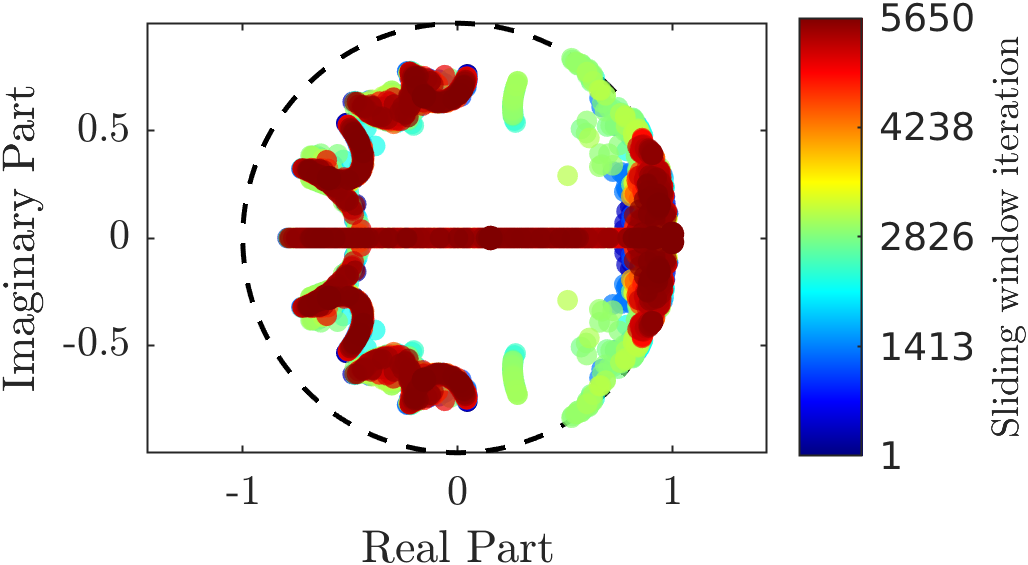}
    \vskip -0.02 true in
    \caption{\footnotesize Evolution of eigenvalue across sliding windows: full spectra with iteration-colored progression.}
    \label{fig:stewart-spectra}
\end{figure}
Together, these results demonstrate that our framework not only denoises real sensor streams but also yields stable  and bounded predictions that are suitable for downstream planning and control applications. A real-time demonstration of the moving-base prediction and MPC integration is shown in the supplementary video~\cite{kombo2025icravideo}.    

\subsection{Parameter Sensitivity and Computational Analysis}\label{sec:ablation}

While Sec.\ref{sec:stewart} validates hardware performance using a selected configuration of the hyper-parameters ($N=250$, $L=10$, $J=20$), we now examine the sensitivity of the prediction accuracy and computational cost to the embedding length $L$ and the number of Cadzow iterations $J$.
\begin{table}[th]
\centering
\captionsetup{font=footnotesize}
\caption{Effect of Cadzow iterations $J$ on runtime and prediction error ($N=250$, $L=10$). Cadzow time corresponds to the iterative Cadzow projection up to a maximum of $J$ iterations, with early termination when the relative tolerance $\texttt{tol}=10^{-6}$ is satisfied.}
\label{tab:cadzow_appendix}
\scriptsize
\setlength{\tabcolsep}{6pt}
\begin{tabular}{c c c c c c c c}
\toprule
$N$ & $L$ & $J$ & 
Cadzow Time (ms) & 
Total Time (ms) & 
RMSE \\
\midrule
250 & 10 & 1  & $0.6772 \pm 0.1673$  & $2.3220 \pm 0.5380$  & $0.012901$ \\
250 & 10 & 5  & $2.2844 \pm 0.0812$  & $4.0004 \pm 0.1495$  & $0.012601$ \\
250 & 10 & 10 & $4.1299 \pm 0.0746$  & $5.8169 \pm 0.1208$  & $0.011727$ \\
250 & 10 & 20 & $7.9089 \pm 0.0642$  & $9.5817 \pm 0.0825$  & $0.011505$ \\
250 & 10 & 30 & $11.7770 \pm 0.5206$ & $13.4480 \pm 0.5879$ & $0.011609$ \\
\bottomrule
\end{tabular}    
\end{table}

\textbf{Cadzow iterations $J$:} Based on Table~\ref{tab:cadzow_appendix}, with $N = 250$ and $L = 10$ fixed, increasing the number of Cadzow iterations improves prediction accuracy up to $J=20$, reducing RMSE from $0.0129$ ($J=1$) to $0.0115$ ($J=20$). Beyond this point, additional iterations provide negligible accuracy gains while increasing computational cost approximately linearly. These results indicate that moderate Cadzow iterations provide a favorable trade-off between accuracy and runtime.

\textbf{Embedding length $L$:} As evidenced in Table~\ref{tab:embedding_ablation} with $N=250$ and $J=20$ fixed, the prediction accuracy improves as the $L$ increases up to an intermediate value. RMSE decreases from $0.060$ at $L=4$ to $0.0108$ at $L=12$, with a slight degradation at $L=13$. This reflects a finite-data trade-off where larger embeddings increase model expressivity but reduce the number of snapshot pairs of $N/L$, which may degrade the matrix conditioning and stability of the locally linear dynamical model. Runtime grows moderately with $L$, with the Cadzow projection constituting the dominant over the total computational time. A comprehensive ablation study analyzing the sensitivity of RMSE, Cadzow runtime, and total computation time to the hyper-parameters is summarized in the Appendix (Table~\ref{tab:ablation_full1}).

\begin{table}[h]
\centering
\captionsetup{font=footnotesize}
\caption{Effect of embedding length $L$ on total computation time and prediction error ($N=250$, $J=20$). Cadzow time corresponds to the iterative Cadzow projection up to a maximum of $J$ iterations, with early termination when the relative tolerance $\texttt{tol}=10^{-6}$ is satisfied.}
\label{tab:embedding_ablation}
\scriptsize
\setlength{\tabcolsep}{6pt}
\begin{tabular}{c c c c c c}
\toprule
$N$ & $L$ & $J$ & 
Cadzow Time (ms) & 
Total Time (ms) & 
RMSE \\
\midrule
250 & 4  & 20 & $7.8140 \pm 3.2639$  & $9.3070 \pm 4.8364$  & $0.060471$ \\
250 & 6  & 20 & $8.1512 \pm 1.8150$  & $9.9204 \pm 2.1636$  & $0.017239$ \\
250 & 8  & 20 & $7.8828 \pm 0.4858$  & $9.5510 \pm 0.5851$  & $0.013038$ \\
250 & 10 & 20 & $7.9628 \pm 0.5314$  & $9.5880 \pm 0.6338$  & $0.011505$ \\
250 & 12 & 20 & $8.1490 \pm 0.5678$  & $9.8960 \pm 0.6983$  & $0.010766$ \\
\bottomrule
\end{tabular}
\end{table}

\section{Conclusion And Future Work} 
\label{sec:conc}

We introduced a Page–Hankel SVHT framework for learning short-horizon predictive models of moving obstacles from noisy data streams. Simulations (Gaussian and heavy-tailed noise) showed good variance-stable, distribution-agnostic performance, and hardware tests with IMU data produce reliable forecasts that can support downstream control. Future work will explore online adaptation, tensorized delayed embeddings for multi-axis consistency, and integration within an MPC-based prediction framework.

\bibliographystyle{IEEEtran}
\bibliography{Bib/refs}
\clearpage
\appendix
\begin{table}[H]
\centering
\caption{Full ablation study over buffer size $N$, embedding length $L$, and Cadzow iterations $J$. Mean runtime per update and prediction RMSE are reported. Parameter combinations that violate the embedding feasibility condition $N \geq L^2$ are omitted from the table.}
\label{tab:ablation_full1}
\scriptsize
\setlength{\tabcolsep}{6pt}
\begin{tabular}{c c c c c c}
\toprule $N$ & $L$ & $J$ & Cadzow Time (ms) & Total Time (ms) & RMSE \\ 
\midrule
120 & 4 & 1  & 0.3832 $\pm$ 0.0699 & 2.1726 $\pm$ 0.4302 & 0.041684 \\
120 & 4 & 5  & 1.1880 $\pm$ 0.0487 & 2.8992 $\pm$ 0.1233 & 0.061402 \\
120 & 4 & 10 & 2.2267 $\pm$ 0.2437 & 3.9636 $\pm$ 0.4341 & 0.066647 \\
120 & 4 & 20 & 4.2195 $\pm$ 0.1328 & 5.9419 $\pm$ 0.2008 & 0.069529 \\
120 & 4 & 30 & 6.2127 $\pm$ 0.1754 & 7.9390 $\pm$ 0.2355 & 0.070924 \\

120 & 6 & 1  & $0.3728 \pm 0.0153$ & $2.0764 \pm 0.0587$ & 0.027131 \\
120 & 6 & 5  & $1.2558 \pm 0.1056$ & $3.0165 \pm 0.2507$ & 0.025013 \\
120 & 6 & 10 & $2.2823 \pm 0.1566$ & $4.0130 \pm 0.2676$ & 0.024226 \\
120 & 6 & 20 & $4.3531 \pm 0.1214$ & $6.0952 \pm 0.1739$ & 0.023255 \\
120 & 6 & 30 & $6.7266 \pm 0.9329$ & $8.5594 \pm 1.2135$ & 0.022915 \\

120 & 8 & 1  & $0.3852 \pm 0.0144$ & $2.1020 \pm 0.0528$ & 0.019806 \\
120 & 8 & 5  & $1.2703 \pm 0.0366$ & $3.0017 \pm 0.0738$ & 0.018468 \\
120 & 8 & 10 & $2.3354 \pm 0.0593$ & $4.0732 \pm 0.1001$ & 0.017621 \\
120 & 8 & 20 & $4.4671 \pm 0.1121$ & $6.2143 \pm 0.1561$ & 0.016687 \\
120 & 8 & 30 & $6.6191 \pm 0.1691$ & $8.3765 \pm 0.2188$ & 0.016714 \\

120 & 10 & 1  & $0.3967 \pm 0.0152$ & $2.2462 \pm 0.0592$ & 0.016142 \\
120 & 10 & 5  & $1.3085 \pm 0.0400$ & $3.1728 \pm 0.0792$ & 0.015643 \\
120 & 10 & 10 & $2.4109 \pm 0.0630$ & $4.2827 \pm 0.1028$ & 0.014795 \\
120 & 10 & 20 & $4.6101 \pm 0.1068$ & $6.4900 \pm 0.1546$ & 0.015392 \\
120 & 10 & 30 & $6.8157 \pm 0.1631$ & $8.7026 \pm 0.2137$ & 0.015196 \\
\midrule
160 & 4 & 1  & $0.4536 \pm 0.0177$ & $2.1603 \pm 0.0641$ & 0.035327 \\
160 & 4 & 5  & $1.5024 \pm 0.0415$ & $3.2191 \pm 0.0820$ & 0.053198 \\
160 & 4 & 10 & $2.8624 \pm 0.2153$ & $4.6078 \pm 0.3510$ & 0.059923 \\
160 & 4 & 20 & $5.4532 \pm 0.1256$ & $7.1880 \pm 0.1725$ & 0.064959 \\
160 & 4 & 30 & $8.2628 \pm 0.6137$ & $10.0730 \pm 0.7768$ & 0.067648 \\

160 & 6 & 1  & $0.5030 \pm 0.1025$ & $2.3446 \pm 0.4554$ & 0.023096 \\
160 & 6 & 5  & $1.6202 \pm 0.1237$ & $3.4279 \pm 0.2614$ & 0.021213 \\
160 & 6 & 10 & $3.3674 \pm 0.6949$ & $5.4237 \pm 1.1209$ & 0.020821 \\
160 & 6 & 20 & $5.9851 \pm 0.8174$ & $7.9146 \pm 1.1202$ & 0.020225 \\
160 & 6 & 30 & $8.6164 \pm 1.0053$ & $10.4830 \pm 1.2799$ & 0.020171 \\

160 & 8 & 1  & $0.4799 \pm 0.0330$ & $2.1996 \pm 0.1358$ & 0.016925 \\
160 & 8 & 5  & $1.5901 \pm 0.0814$ & $3.3196 \pm 0.1740$ & 0.016132 \\
160 & 8 & 10 & $2.9408 \pm 0.1819$ & $4.6836 \pm 0.3040$ & 0.015741 \\
160 & 8 & 20 & $5.7999 \pm 0.5617$ & $7.5973 \pm 0.7701$ & 0.014883 \\
160 & 8 & 30 & $8.6106 \pm 1.1035$ & $10.4320 \pm 1.3751$ & 0.014916 \\

160 & 10 & 1  & $0.5095 \pm 0.0648$ & $2.4343 \pm 0.2764$ & 0.013873 \\
160 & 10 & 5  & $1.6393 \pm 0.0952$ & $3.5114 \pm 0.2040$ & 0.013687 \\
160 & 10 & 10 & $3.1312 \pm 0.2806$ & $5.0611 \pm 0.4773$ & 0.012882 \\
160 & 10 & 20 & $6.0941 \pm 0.6909$ & $8.0680 \pm 0.9624$ & 0.012936 \\
160 & 10 & 30 & $8.6109 \pm 0.4076$ & $10.4910 \pm 0.5188$ & 0.013258 \\

160 & 12 & 1  & $0.5086 \pm 0.0460$ & $2.4802 \pm 0.1959$ & 0.013217 \\
160 & 12 & 5  & $1.7179 \pm 0.1428$ & $3.7125 \pm 0.3134$ & 0.013185 \\
160 & 12 & 10 & $3.4554 \pm 0.5700$ & $5.6472 \pm 0.9401$ & 0.012726 \\
160 & 12 & 20 & $6.4075 \pm 0.5559$ & $8.5446 \pm 0.7773$ & 0.012999 \\
160 & 12 & 30 & $9.6981 \pm 1.1509$ & $11.9160 \pm 1.4566$ & 0.013409 \\
\midrule
200 & 4 & 1  & $0.6268 \pm 0.1352$ & $2.5634 \pm 0.5278$ & 0.032814 \\
200 & 4 & 5  & $2.0187 \pm 0.1987$ & $3.9319 \pm 0.4027$ & 0.046911 \\
200 & 4 & 10 & $3.6197 \pm 0.4050$ & $5.4874 \pm 0.6579$ & 0.052967 \\
200 & 4 & 20 & $6.4935 \pm 0.4080$ & $8.2262 \pm 0.5494$ & 0.057141 \\
200 & 4 & 30 & $9.4777 \pm 0.3388$ & $11.1760 \pm 0.3938$ & 0.059086 \\

200 & 6 & 1  & $0.5512 \pm 0.0105$ & $2.2510 \pm 0.0282$ & 0.021260 \\
200 & 6 & 5  & $1.8868 \pm 0.1215$ & $3.6242 \pm 0.2418$ & 0.018846 \\
200 & 6 & 10 & $3.4322 \pm 0.0458$ & $5.1503 \pm 0.0629$ & 0.018136 \\
200 & 6 & 20 & $6.6163 \pm 0.0732$ & $8.3405 \pm 0.0887$ & 0.017013 \\
200 & 6 & 30 & $9.7807 \pm 0.1197$ & $11.5080 \pm 0.1366$ & 0.016719 \\

200 & 8 & 1  & $0.5607 \pm 0.0106$ & $2.2701 \pm 0.0257$ & 0.015659 \\
200 & 8 & 5  & $1.8957 \pm 0.0287$ & $3.6222 \pm 0.0400$ & 0.014751 \\
200 & 8 & 10 & $3.5992 \pm 0.2598$ & $5.3798 \pm 0.4134$ & 0.013911 \\
200 & 8 & 20 & $7.3050 \pm 0.4703$ & $9.2472 \pm 0.6550$ & 0.013121 \\
200 & 8 & 30 & $10.3450 \pm 1.1833$ & $12.1630 \pm 1.4163$ & 0.012931 \\

200 & 10 & 1  & $0.5700 \pm 0.0182$ & $2.3007 \pm 0.0615$ & 0.012736 \\
200 & 10 & 5  & $1.9149 \pm 0.1157$ & $3.6880 \pm 0.2079$ & 0.012313 \\
200 & 10 & 10 & $3.5207 \pm 0.0820$ & $5.2712 \pm 0.1327$ & 0.011537 \\
200 & 10 & 20 & $6.7313 \pm 0.1696$ & $8.4971 \pm 0.2418$ & 0.011565 \\

\bottomrule 
\end{tabular} 
\end{table} 

\begin{table}[!t]
\centering 
\label{tab:ablation_full2} 
\scriptsize 
\setlength{\tabcolsep}{6pt} 
\begin{tabular}{c c c c c c} 
\toprule $N$ & $L$ & $J$ & Cadzow Time (ms) & Total Time (ms) & RMSE \\
\midrule
200 & 10 & 30 & $10.2540 \pm 1.1220$ & $12.1300 \pm 1.3003$ & 0.011832 \\

200 & 12 & 1  & $0.5946 \pm 0.0545$ & $2.4552 \pm 0.1724$ & 0.011933 \\
200 & 12 & 5  & $1.9736 \pm 0.1050$ & $3.8558 \pm 0.1780$ & 0.011858 \\
200 & 12 & 10 & $3.6454 \pm 0.1021$ & $5.5318 \pm 0.1534$ & 0.011015 \\
200 & 12 & 20 & $6.9299 \pm 0.3702$ & $8.8380 \pm 0.4753$ & 0.011267 \\
200 & 12 & 30 & $10.2400 \pm 0.2077$ & $12.1330 \pm 0.2815$ & 0.011693 \\

\midrule
250 & 4 & 1  & $0.6227 \pm 0.0335$ & $2.0243 \pm 0.0759$ & 0.035939 \\
250 & 4 & 5  & $2.0924 \pm 0.0602$ & $3.5308 \pm 0.1049$ & 0.050362 \\
250 & 4 & 10 & $3.8997 \pm 0.0895$ & $5.3408 \pm 0.1354$ & 0.056567 \\
250 & 4 & 20 & $7.6126 \pm 0.3603$ & $9.1088 \pm 0.4394$ & 0.060471 \\
250 & 4 & 30 & $11.4740 \pm 0.3172$ & $13.0500 \pm 0.3914$ & 0.062084 \\

250 & 6 & 1  & $0.6576 \pm 0.0459$ & $2.3620 \pm 0.1468$ & 0.022799 \\
250 & 6 & 5  & $2.2268 \pm 0.0740$ & $3.9937 \pm 0.1471$ & 0.019384 \\
250 & 6 & 10 & $4.0392 \pm 0.0825$ & $5.7479 \pm 0.1180$ & 0.018490 \\
250 & 6 & 20 & $8.0214 \pm 0.2732$ & $9.8137 \pm 0.3729$ & 0.017239 \\
250 & 6 & 30 & $11.9460 \pm 0.3153$ & $13.7950 \pm 0.4033$ & 0.016810 \\

250 & 8 & 1  & $0.6638 \pm 0.0318$ & $2.3627 \pm 0.0763$ & 0.016137 \\
250 & 8 & 5  & $2.2586 \pm 0.0899$ & $4.0174 \pm 0.1413$ & 0.014821 \\
250 & 8 & 10 & $4.2076 \pm 0.1442$ & $5.9728 \pm 0.2176$ & 0.013984 \\
250 & 8 & 20 & $8.1604 \pm 0.2554$ & $9.9737 \pm 0.3562$ & 0.013038 \\
250 & 8 & 30 & $12.1310 \pm 0.6385$ & $13.9770 \pm 0.7098$ & 0.012717 \\

250 & 10 & 1  & $0.6882 \pm 0.0388$ & $2.3959 \pm 0.1036$ & 0.012901 \\
250 & 10 & 5  & $2.3419 \pm 0.0939$ & $4.1217 \pm 0.1778$ & 0.012601 \\
250 & 10 & 10 & $4.3700 \pm 0.1337$ & $6.1755 \pm 0.2236$ & 0.011727 \\
250 & 10 & 20 & $8.4571 \pm 0.3276$ & $10.3050 \pm 0.4068$ & 0.011505 \\
250 & 10 & 30 & $12.3650 \pm 0.3129$ & $14.1810 \pm 0.3991$ & 0.011609 \\

250 & 12 & 1  & $0.7249 \pm 0.0505$ & $2.6276 \pm 0.1058$ & 0.011971 \\
250 & 12 & 5  & $2.4146 \pm 0.1028$ & $4.3520 \pm 0.2096$ & 0.011652 \\
250 & 12 & 10 & $4.5296 \pm 0.1474$ & $6.4944 \pm 0.2628$ & 0.010894 \\
250 & 12 & 20 & $8.7088 \pm 0.2401$ & $10.6870 \pm 0.3421$ & 0.010766 \\
250 & 12 & 30 & $12.7920 \pm 0.4162$ & $14.7800 \pm 0.5487$ & 0.010877 \\

250 & 15 & 1  & $0.7355 \pm 0.0437$ & $2.6830 \pm 0.1201$ & 0.011677 \\
250 & 15 & 5  & $2.5120 \pm 0.0906$ & $4.5119 \pm 0.1665$ & 0.013119 \\
250 & 15 & 10 & $4.7152 \pm 0.1292$ & $6.7621 \pm 0.2094$ & 0.010656 \\
250 & 15 & 20 & $9.0523 \pm 0.3065$ & $11.1050 \pm 0.4194$ & 0.010824 \\
250 & 15 & 30 & $13.0690 \pm 0.3104$ & $15.0780 \pm 0.4002$ & 0.011014 \\
\midrule
300 & 4  & 1  & $0.7228 \pm 0.0475$  & $2.1547 \pm 0.0862$  & 0.035891 \\
300 & 4  & 5  & $2.4456 \pm 0.0619$  & $3.9187 \pm 0.1027$  & 0.050234 \\
300 & 4  & 10 & $4.5938 \pm 0.2663$  & $6.0738 \pm 0.3200$  & 0.056825 \\
300 & 4  & 20 & $8.8316 \pm 0.1681$  & $10.3190 \pm 0.2177$ & 0.061132 \\
300 & 4  & 30 & $13.1310 \pm 0.1631$ & $14.6460 \pm 0.1973$ & 0.063097 \\

300 & 6  & 1  & $0.7340 \pm 0.0299$  & $2.3591 \pm 0.0729$  & 0.023009 \\
300 & 6  & 5  & $2.5173 \pm 0.0581$  & $4.2085 \pm 0.0929$  & 0.018807 \\
300 & 6  & 10 & $4.6965 \pm 0.0843$  & $6.3929 \pm 0.1214$  & 0.017704 \\
300 & 6  & 20 & $9.0858 \pm 0.2819$  & $10.7900 \pm 0.3144$ & 0.016135 \\
300 & 6  & 30 & $13.3950 \pm 0.2552$ & $15.1130 \pm 0.2863$ & 0.015616 \\

300 & 8  & 1  & $0.7423 \pm 0.0265$  & $2.3620 \pm 0.0696$  & 0.016342 \\
300 & 8  & 5  & $2.5460 \pm 0.0634$  & $4.2264 \pm 0.0965$  & 0.014672 \\
300 & 8  & 10 & $4.7722 \pm 0.0916$  & $6.4750 \pm 0.1337$  & 0.013866 \\
300 & 8  & 20 & $9.1867 \pm 0.1419$  & $10.9120 \pm 0.1681$ & 0.013119 \\
300 & 8  & 30 & $13.6910 \pm 0.1940$ & $15.4220 \pm 0.2280$ & 0.012822 \\

300 & 10 & 1  & $0.7727 \pm 0.0302$  & $2.4287 \pm 0.0692$  & 0.012846 \\
300 & 10 & 5  & $2.6267 \pm 0.0650$  & $4.3337 \pm 0.1003$  & 0.012500 \\
300 & 10 & 10 & $4.9612 \pm 0.0886$  & $6.6757 \pm 0.1236$  & 0.011835 \\
300 & 10 & 20 & $9.5891 \pm 0.1361$  & $11.3260 \pm 0.1747$ & 0.011621 \\
300 & 10 & 30 & $14.1600 \pm 0.2190$ & $15.9070 \pm 0.2569$ & 0.011731 \\

300 & 12 & 1  & $0.8090 \pm 0.0356$  & $2.5198 \pm 0.0765$  & 0.011694 \\
300 & 12 & 5  & $2.7481 \pm 0.0598$  & $4.5100 \pm 0.0973$  & 0.011709 \\
300 & 12 & 10 & $5.1364 \pm 0.0972$  & $6.8915 \pm 0.1337$  & 0.011250 \\
300 & 12 & 20 & $9.8436 \pm 0.1819$  & $11.6140 \pm 0.2178$ & 0.011282 \\
300 & 12 & 30 & $14.6930 \pm 0.3004$ & $16.4770 \pm 0.3285$ & 0.011323 \\

300 & 15 & 1  & $0.8220 \pm 0.0368$  & $2.6515 \pm 0.0764$  & 0.011138 \\
300 & 15 & 5  & $2.8201 \pm 0.0738$  & $4.6866 \pm 0.1221$  & 0.013286 \\
300 & 15 & 10 & $5.3040 \pm 0.1636$  & $7.2028 \pm 0.2130$  & 0.011434 \\
300 & 15 & 20 & $10.1500 \pm 0.2119$ & $12.0700 \pm 0.2477$ & 0.011694 \\
300 & 15 & 30 & $15.0820 \pm 0.5309$ & $17.0180 \pm 0.5609$ & 0.011322 \\

\bottomrule 
\end{tabular} 
\end{table}

\end{document}